\documentclass[runningheads]{llncs}

\usepackage[T1]{fontenc}
\usepackage{graphicx}
\usepackage{tikz}
\usepackage{booktabs}
\usepackage[breaklinks, colorlinks, linkcolor=blue, citecolor=blue, urlcolor=black]{hyperref}
\makeatletter
\def\orcidID#1{\href{http://orcid.org/#1}{\protect\raisebox{-1.25pt}{\protect\includegraphics{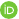}}}}
\makeatother

\graphicspath{{./figs/}}

\usepackage{xspace}
\usepackage{amssymb,amsmath}
\usepackage{esvect}

\newcommand\litell{\ell}
\newcommand\varx{x\xspace}
\newcommand\olnot[1]{\overline{#1}\xspace}
\newcommand{\coeffa}{a\xspace}

\newcommand{\constrC}{C\xspace}
\newcommand\constrc{\constrC}

\newcommand{\formF}{F\xspace}
\newcommand{\formula}{\formF}
\newcommand{\formf}{\formF}
\newcommand{\objective}{O}
\newcommand{\objs}{\mathbf{O}}

\newcommand\subst{\omega}
\newcommand{\assmnt}{\alpha}
\newcommand{\assmntone}{\alpha}
\newcommand\assmnttwo{\beta}
\newcommand{\sumnodisplay}{{\textstyle \sum}}

\newcommand\under[1]{\!\!\upharpoonright_{\!#1}}
\newcommand\lowerbound{L}
\newcommand\ignore[1]{}
\newcommand\weaklyDominates{\preceq_P}
\newcommand\dominates{\prec_P}

\usepackage{xcolor}
\providecommand\citet[1]{{\color{red}Authors?}~\cite{#1}}

\newcommand\veripb{\textsc{VeriPB}\xspace}
\newcommand\coreset{\mathcal{C}}
\newcommand\derivedset{\mathcal{D}}
\newcommand\order{\mathcal{O}}
\newcommand\varsu{\vv{u}}\newcommand\varsx{\vv{x}}
\newcommand\varsv{\vv{v}}\newcommand\varsw{\vv{w}}
\newcommand\varsz{\vv{z}} 
\newcommand\lit{\ell}
\newcommand\degree{A}
\newcommand\lequiv\Leftrightarrow
\newcommand\proofVeriPB{P}

\newcommand\pminimal{$P$-minimal\xspace}
\newcommand\pmin{\pminimal}
\newcommand\bioptsat{\textsc{BiOptSat}\xspace}
\newcommand\bos{\bioptsat}
\newcommand\lowerbounding{\textsc{LowerBound}\xspace}

\newcommand\eg{e.g.,\xspace}

\newcommand\citepmin{\cite{KoshimuraEtAl2009MinimalModelGeneration,SohEtAl2017SolvingMultiobjectiveDiscrete}\xspace}
\newcommand\citelb{\cite{CortesEtAl2023NewCoreGuided}\xspace}
\newcommand\citebos{\cite{JabsEtAl2024FromSingleObjective}\xspace}

\newcommand\nth[1]{$#1^{\text{th}}$\xspace }

\newcommand\limpliedby{\Leftarrow} 
\newcommand\limplies{\Rightarrow} 

\newcommand\mininc{\texttt{MinInc}\xspace}
\newcommand\mindec{\texttt{MinDec}\xspace}

\newcommand\pdcut{PD cut\xspace}
\newcommand\pdcuts{PD cuts\xspace}
\newcommand\worse{w} 

\begin{document}
\title{Certifying Pareto-Optimality in\\Multi-Objective Maximum Satisfiability}
\author{Christoph Jabs\inst{1}\orcidID{0000-0003-3532-696X} \and
Jeremias Berg\inst{1}\orcidID{0000-0001-7660-8061}\and
Bart Bogaerts\inst{2,3}\orcidID{0000-0003-3460-4251} \and
Matti J\"arvisalo\inst{1}\orcidID{0000-0003-2572-063X}
}
\authorrunning{Jabs et al.}
\institute{University of Helsinki, Helsinki, Finland \\
\email{\{christoph.jabs, jeremias.berg, matti.jarvisalo\}@helsinki.fi}\\
\and
KU Leuven, Leuven, Belgium\\
\email{bart.bogaerts@kuleuven.be}\\
\and 
Vrije Universiteit Brussel, Brussels, Belgium}
\maketitle              %
\begin{abstract}
  Due to the wide employment of automated reasoning in
  the analysis and construction of correct systems, the results reported
  by automated reasoning engines must be trustworthy.
  For Boolean satisfiability (SAT) solvers---and more recently
  SAT-based maximum satisfiability (MaxSAT) solvers---trustworthiness is 
  obtained by integrating proof logging into solvers, making solvers
  capable of emitting machine-verifiable proofs to certify correctness of the reasoning steps performed.
  In this work, we enable for the first time
  proof logging based on the \veripb proof format for  multi-objective MaxSAT (MO-MaxSAT) optimization
  techniques.
 Although \veripb does not offer direct support for multi-objective problems,
 we detail how preorders in \veripb can be used
to provide certificates for               
MO-MaxSAT algorithms computing a representative solution for each element in the non-dominated
 set of the search space under Pareto-optimality,
 without extending the \veripb format or the proof checker.
By implementing \veripb proof logging into a state-of-the-art multi-objective MaxSAT solver,
we show empirically that proof logging can be made scalable for MO-MaxSAT with reasonable overhead.

\end{abstract}

\section{Introduction}

Automated reasoning is central in enabling the analysis and construction of correct systems.
Practical solvers developed in the realm of automated reasoning,
such as Boolean satisfiability (SAT) solvers~\cite{BHvW21HandbookSatisfiability-SecondEdition}, facilitate the development of 
more complex automated reasoning systems. One
successful example of such generic SAT-based approaches are 
solvers developed for maximum satisfiability~(MaxSAT)~\cite{BJM21MaximumSatisfiabiliy}---the optimization extension of SAT---enabling solving various NP-hard real-world optimization problems~\cite{BJM21MaximumSatisfiabiliy}.
Further, SAT-based approaches are being generalized and
developed for MaxSAT under multiple objectives, i.e., multi-objective MaxSAT~\cite{DBLP:conf/sat/CabralJM22,CortesEtAl2024SlideDrill,CortesEtAl2023NewCoreGuided,DBLP:journals/cor/GuerreiroCVBLMF23,JBIJ23PreprocessingSAT-BasedMulti-ObjectiveCombinatorialOptimization,JBJ24CoreBoostingSAT-BasedMulti-objectiveOptimization,JabsEtAl2024FromSingleObjective,SohEtAl2017SolvingMultiobjectiveDiscrete,TerraNevesEtAl2018MultiObjectiveOptimization}, with the aim
of extending the success of MaxSAT to more efficiently solving real-world multi-objective optimization
problems, from, e.g., staff scheduling~\cite{DBLP:journals/anor/DemirovicMW19}
through package upgradeability~\cite{DBLP:journals/jsat/JanotaLMM12} to finding
interpretable classifiers~\cite{DBLP:conf/cp/MaliotovM18}.

The more SAT and MaxSAT solvers are used in real-world settings, the more important it is to be able to
trust the results solvers provide. While solutions are generally easy to confirm,
solvers should be trustworthy also
when they report unsatisfiability or, in the context of optimization, when the solvers claim that a
solution is optimal and hence no better solutions exist.
In response to these concerns, proof logging and checking techniques for SAT solvers have been
developed and widely
adopted~\cite{GN03VerificationProofsUnsatisfiabilityCNFFormulas,CHHKS17EfficientCertifiedRATVerification,CMS17EfficientCertifiedResolutionProofChecking},
among which
DRAT~\cite{HHW13Trimmingwhilecheckingclausalproofs,HHW13VerifyingRefutationsExtendedResolution}
remains today the de facto standard in the context of SAT solving.
However, DRAT and other SAT proof formats work purely on propositional clauses, which
makes them unsuitable for proof logging MaxSAT solvers. Instead, the
 \veripb
format~\cite{GN21CertifyingParityReasoningEfficientlyUsingPseudo-Boolean,BGMN23CertifiedDominanceSymmetryBreakingCombinatorialOptimisation}, which is based on 
pseudo-Boolean constraints (i.e., 0-1 linear inequalities) and
offers direct support for reasoning about %
objective values in
single-objective optimization problems, has enabled %
proof
logging for %
various optimization
contexts~\cite{BBNOV23CertifiedCore-GuidedMaxSATSolving,BGMN23CertifiedDominanceSymmetryBreakingCombinatorialOptimisation,DMMNOS24Pseudo-BooleanReasoningAboutStatesTransitionsCertify,EGMN20JustifyingAllDifferencesUsingPseudo-BooleanReasoning,GMNO22CertifiedCNFTranslationsPseudo-BooleanSolving,GMMNPT20CertifyingSolversCliqueMaximumCommonConnected,GMN20SubgraphIsomorphismMeetsCuttingPlanesSolving,GMN22AuditableConstraintProgrammingSolver,GN21CertifyingParityReasoningEfficientlyUsingPseudo-Boolean,HOGN24CertifyingMIP-BasedPresolveReductions0-1Integer,IOTBJMN24CertifiedMaxSATPreprocessing,VDB22QMaxSATpbCertifiedMaxSATSolver},
including MaxSAT solving~\cite{BBNOPV24CertifyingWithoutLossGeneralityReasoningSolution-Improving,BBNOV23CertifiedCore-GuidedMaxSATSolving,IOTBJMN24CertifiedMaxSATPreprocessing,msc/Vandesande23,VDB22QMaxSATpbCertifiedMaxSATSolver}.

In this work, we enable proof logging for various recently-proposed multi-objective MaxSAT
solving techniques. To the best of our knowledge, this is the first work enabling proof logging
in multi-objective optimization.%
\footnote{Although relatively distant to the present work, there is some work on certificates in the context of multi-objective queries in Markov decision processes~\cite{DBLP:conf/qestformats/BaierCK24}.}
Our solution builds on the \veripb format.
 It is critical to note that
\veripb does not offer direct support for multiple objective functions,
and is thereby seemingly restricted to proof logging single-objective
optimization algorithms.
However, as we will detail, proof logging for MO-MaxSAT can in fact be enabled without extending
the \veripb format or the proof checker. In particular, in order to provide certificates for
MO-MaxSAT algorithms developed for computing a representative solution for each element in the so-called non-dominated
 set of the search space under Pareto-optimality~\cite{E05MulticriteriaOptimization2ed}, we make in a specific way
use of preorders supported by \veripb. While preorders were
first introduced to \veripb for certifying symmetry and dominance
breaking~\cite{BGMN23CertifiedDominanceSymmetryBreakingCombinatorialOptimisation}, here we show that, in fact,
a single preorder suffices for certifying that an MO-MaxSAT algorithm has 
computed a representative solution at each element of the non-dominated set.
As representative MO-MaxSAT techniques, we detail \veripb-based proof logging for variants of
the \pminimal~\citepmin, \bioptsat~\citebos, and \lowerbounding~\citelb  approaches,
as well as the recently-proposed MO-MaxSAT preprocessing/reformulation
technique of core boosting~\cite{JBJ24CoreBoostingSAT-BasedMulti-objectiveOptimization} which has been shown to provide
considerable runtime improvements to MO-MaxSAT solvers.
By adding \veripb proof logging to the implementations of these approaches in the MO-MaxSAT solver Scuttle~\cite{scuttle,JBIJ23PreprocessingSAT-BasedMulti-ObjectiveCombinatorialOptimization,JBJ24CoreBoostingSAT-BasedMulti-objectiveOptimization}, we show empirically that proof logging can be made scalable for MO-MaxSAT, with average
proof logging overhead ranging from 14\% to 29\% depending on the solving approach.

\section{Preliminaries}

We begin with necessary preliminaries related to multi-objective MaxSAT and VeriPB proofs.

\subsection{Clauses and Pseudo-Boolean Constraints}

A literal $\litell$ is a $\{0,1\}$-valued Boolean variable $\varx$ or its negation $\olnot{\varx} \equiv 1-\varx$. 
A propositional clause $\constrC = (\litell_1 \lor \ldots \lor \litell_k)$ is a disjunction of literals. A formula in conjunctive 
normal form (CNF)
$\formula = \constrC_1 \land \ldots \land \constrC_m$ is a conjunction of clauses. We often think of clauses as sets of literals and 
formulas as sets of clauses.

A (normalized)
pseudo-Boolean (PB) constraint is a 0-1 linear inequality $\constrC = \sumnodisplay_{i} \coeffa_i \litell_i \geq b$ where \(\coeffa_i\) are positive integers and \(b\) a non-negative integer.
We will assume wlog that all PB constraints are in \emph{normal form}, meaning that the $\litell_i$ are over different variables and all coefficients $\coeffa_i$ and the bound are positive. A pseudo-Boolean (PB) formula 
is a conjunction (or set) of PB constraints. We identify the propositional clause $\constrC = (\litell_1 \lor \ldots \lor \litell_k)$
with the PB constraint $\sumnodisplay_{i=1}^k \litell_i \geq 1$. %
This is convenient as the algorithms and 
solvers that we develop proof logging for expect propositional clauses as input, and---as detailed in this work---produce their proofs in pseudo-Boolean format. 
If $\constrc$ is the PB constraint $\sum_i\coeffa_i\litell_i\geq b$, we write \(\neg C\) for its negation \(\sum_i\coeffa_i\olnot{\litell_i} \ge \sum_i\coeffa_i - b + 1\).
If $p$ is furthermore a variable, we write $p\lequiv \constrc$ for the two constraints
expressing that \(p\) implies \(\constrc\) and vice versa, i.e., 
\(Mp + \sum_i\coeffa_i\olnot{\litell_i}\geq M\) and \(b \olnot{p} + \sum_i\coeffa_i\litell_i \geq b\) where \(M=\sum_i\coeffa_i -b +1\).
An \emph{objective} $\objective$ is an expression $\sum_i \coeffa_i \litell_i+\lowerbound$ where the $\coeffa_i$ and $\lowerbound$ are integers. 

A substitution \(\subst\) maps each variable in its domain to a truth value (either $0$ or $1$) or to another literal.
We denote by $\constrc\under\subst$  the constraint obtained from $\constrc$ by replacing each variable $\varx$ in the domain of $\subst$ by $\subst(\varx)$; the notations
$\objective\under\subst$, $\formf\under\subst$, and \(\varsx\under\subst\) for a tuple of variables \(\varsx\) are defined analogously. 
An assignment \(\assmnt\) is a substitution that maps only onto $\{0,1\}$.
When convenient, we view an assignment as the set of literals it sets to $1$.
An assignment $\assmnt$ is \emph{complete} for a constraint, formula, or objective if $\assmnt$
maps each variable in them to a value, and partial otherwise. 
The assignment $\assmnt$ satisfies a constraint $\constrc$ if the constraint $\constrc\under\assmnt = \sumnodisplay_{i} \coeffa'_i \litell'_i \geq b'$ obtained after normalization has $b'=0$, and falsifies $\constrc$ if $\sumnodisplay_{i} \coeffa'_i < b'$. In other words, $\assmnt$ satisfies $\constrc$ if simplifying $\constrc$ by $\assmnt$ leads to a trivial constraint and falsifies $\constrc$ if no extension of  $\assmnt$ satisfies $\constrc$. An assignment $\assmnt$ is a solution to a formula $\formula$ if $\assmnt$ satisfies all constraints in $\formula$.
A constraint $\constrc$ is \emph{implied} by $\formula$ (denoted by $\formula\models\constrc$) if all solutions of $\formula$ also satisfy $\constrc$.

\subsection{Multi-Objective MaxSAT}

An instance $(\formula, \objs)$ of multi-objective MaxSAT (MO-MaxSAT) consists of a CNF formula $\formula$ and a set $\objs = (\objective_1, \ldots, \objective_p)$ 
of $p$ objectives under minimization. This definition for MO-MaxSAT captures standard (single-objective, weighted partial) MaxSAT by setting $p=1$; see, e.g., \citebos.
 Given two assignments $\assmntone$ and~$\assmnttwo$ that are complete for each $\objective_i$, we say that 
$\assmntone$ \emph{weakly dominates} $\assmnttwo$ 
(and write $\assmntone\weaklyDominates\assmnttwo$) if $\objective_i\under{\assmntone} \leq \objective_i\under{\assmnttwo}$ holds for each $i=1,\dots,p$. Note that $\objective\under{\assmnt}$ is an integer value when $\assmnt$ is complete for $\objective$.  If additionally  $\objective_t\under{\assmntone} < \objective_t\under{\assmnttwo}$ for some \(t\), we say that $\assmntone$ \emph{dominates} $\assmnttwo$ (and write $\assmntone\dominates\assmnttwo$). A solution $\assmnt$ to $\formula$ is Pareto-optimal for \((\formula, \objs)\) if \(\assmnt\) is not dominated by any other solution to $\formula$. The \emph{non-dominated set} of \((\formula, \objs)\) consists of the tuples of objective values of Pareto-optimal solutions, i.e., $\{ (\objective_1\under\assmnt, \ldots, \objective_p\under\assmnt) \mid \assmnt \mbox{ is Pareto-optimal for } (\formula, \objs)\}$. Note that each element in the non-dominated set can correspond to several Pareto-optimal solutions. We focus on the task of computing a representative solution for each element of the non-dominated set. 

\begin{figure}[t]
  \begin{minipage}{.43\textwidth}
  \begin{align*}
    \formula =& (\varx_1 \lor \varx_2 \lor \varx_3) \land (\varx_1 \lor \varx_2 \lor \varx_4) \\
              & \land (\varx_2 \lor \varx_3 \lor \varx_5) \land (\varx_3 \lor \varx_4 \lor \varx_5) \\
    \objective_1 &= 3 \varx_2 + 4 \varx_3 + 2 \varx_4 + 5 \varx_5 \\
    \objective_2 &= 7 \varx_1 + 4 \varx_2 + 1 \varx_3 + 2 \varx_4
  \end{align*}
  \end{minipage}
  \;
  \begin{minipage}{.52\textwidth}
    \centering
    \includegraphics{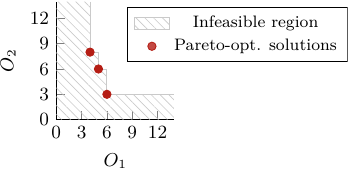}
  \end{minipage}
  \caption{A bi-objective MaxSAT instance and its Pareto-optimal solutions.}\label{fig:running-ex}
\end{figure}

\begin{example}
  Consider the bi-objective MaxSAT instance
  \((\formula,(\objective_1,\objective_2))\)  in
  Fig.~\ref{fig:running-ex}. Its non-dominated set is
  \(\{(4,8),(5,6),(6,3)\}\).  Its three Pareto-optimal solutions are
  \(\{\varx_1,\olnot{\varx_2},\varx_3,\olnot{\varx_4},\olnot{\varx_5}\}\),
  \(\{\olnot{\varx_1},\varx_2,\olnot{\varx_3},\varx_4,\olnot{\varx_5}\}\), and
  \(\{\olnot{\varx_1},\olnot{\varx_2},\varx_3,\varx_4,\olnot{\varx_5}\}\).
\end{example}

\subsection{Multi-Objective MaxSAT Solving}                                                                                                   
                                                                                                                                 
We consider the multi-objective MaxSAT problem of finding a representative
solution of each element in the non-dominated set. Various algorithms for this
problem setting have been proposed
recently~\cite{CortesEtAl2024SlideDrill,CortesEtAl2023NewCoreGuided,JabsEtAl2024FromSingleObjective,SohEtAl2017SolvingMultiobjectiveDiscrete,TerraNevesEtAl2018MultiObjectiveOptimization}.
These algorithms make incremental use of a SAT
solver~\cite{ES03ExtensibleSAT-solver,MLM21Conflict-DrivenClauseLearningSATSolvers} while adding
constraints to the working formula, ending with an unsatisfiable working formula
once all elements in the non-dominated set  have been discovered.

For many of the existing algorithms for MO-MaxSAT, a crucial building block is
what we call a Pareto-dominance cut, or \pdcut for short. A \pdcut is a
(set of) constraint(s) that, given a solution \(\assmnt\), is falsified exactly by all
solutions that are weakly dominated by \(\assmnt\) (including \(\assmnt\) itself).
Adding a \pdcut to the working formula therefore excludes solutions weakly dominated by \(\assmnt\)
from further consideration. Note that for the single-objective case
(\(p=1\)), a \pdcut is identical to a solution-improving constraint, admitting only ``better'' solutions. 
This no longer holds when $p>1$ since Pareto-dominance is not a total order: solutions that are incomparable to $\alpha$ will satisfy the \pdcut. 

For objectives \(\objs=(\objective_1,\dots,\objective_p)\) and solution
\(\assmnt\), let \(\worse_i\) be fresh variables for each \(\objective_i\) with their semantics defined by
\begin{equation}\label{eq:worsedef}
  \worse_i \lequiv \objective_i \geq \objective_i\under\assmnt.
\end{equation}
The \pdcut is the clause \((\olnot{\worse_1} \lor \dots \lor \olnot{\worse_p})\). For
encoding~\eqref{eq:worsedef}, MO-MaxSAT algorithms use a choice of various CNF encodings~\cite{BailleuxBoufkhad2003EfficientCNFEncoding,EenSoerensson2006TranslatingPseudoBoolean,JoshiEtAl2015GeneralizedTotalizerEncoding,KP19Encodingcardinalityconstraintsusingmultiwaymerge,MartinsEtAl2014IncrementalCardinalityConstraints}.
We will return in more detail to how \pdcuts are employed in the \pminimal~\citepmin,
\lowerbounding~\citelb, and \bioptsat~\citebos algorithms
in Section~\ref{sec:main} when detailing \veripb 
proof logging for each of these algorithms.

\subsection{\veripb}\label{sec:proofsystem}

We now overview a  simplified version of the \veripb proof system, only discussing the rules that are relevant for our current exposition. 
For instance, while \veripb supports \emph{single-objective} optimization, we will only use the decision version of this proof system.  
We refer the interested reader to earlier work~\cite{BGMN23CertifiedDominanceSymmetryBreakingCombinatorialOptimisation,GN21CertifyingParityReasoningEfficientlyUsingPseudo-Boolean} for an exposition of the full proof system.

Given a PB input formula $\formula$, the \veripb proof system
maintains a \emph{proof configuration} $\langle\coreset,\derivedset,\order,\varsz\rangle$, consisting of two sets of constraints, the \emph{core set} $\coreset$ and the \emph{derived set} $\derivedset$, a pseudo-Boolean formula $\order(\varsu,\varsv)$ over two tuples of variables $\varsu,\varsv$ that do not appear in $\coreset$, and a tuple of variables $\varsz$. 
The core set can be thought of as being equal
to $\formula$ and the derived set as consisting of all constraints derived in the proof.
The order $\order$ defines a preorder $\preceq$ on assignments as follows. 
If $\alpha$ and $\beta$ are assignments, then $\alpha\preceq\beta$ iff $\order(\varsz\under\alpha,\varsz\under\beta)$ is true. 
The proof system will guarantee that  $\preceq$ is indeed a preorder, i.e., a reflexive and transitive relation.
The preorder in this configuration was originally introduced in the context of symmetry and dominance breaking~\cite{BGMN23CertifiedDominanceSymmetryBreakingCombinatorialOptimisation}. Here we will use the preorder for a different purpose and will in fact not use the dominance rule introduced in~\cite{BGMN23CertifiedDominanceSymmetryBreakingCombinatorialOptimisation}. 
The precise role of the preorder in our proofs for MO-MaxSAT will be detailed in Section~\ref{sec:main}. 

The configuration is initialized by setting $\coreset=\formula$, $\derivedset=\emptyset$, $\order=\emptyset$ (the empty, and hence trivially true formula) and $\varsz=()$, the empty tuple. 
Afterwards, the configuration is updated using the rules detailed next.

New constraints can be added to $\derivedset$ by deriving them from previously derived constraints in $\coreset\cup\derivedset$
using the
\emph{cutting planes} proof system~\cite{CCT87complexitycutting-planeproofs} consisting of the following rules.
\begin{description}
	\item [Literal Axioms.]
	For any literal~$\lit$, $\lit \geq 0$ is an axiom and can be derived.
	
	\item [Linear Combination.]
	Any positive integer linear combination of two previously derived PB constraints can be inferred. 	
	\item [Division.]
	Given the normalized PB constraint $ \sum_i w_i \lit_i \geq \degree$
	and a positive integer~$c$, 	the constraint
	$ \sum_i \lceil w_i/c \rceil \lit_i \geq \lceil
	\degree/c \rceil$
	can be inferred.
\end{description}
Conveniently, \veripb also allows adding an implied constraint without giving an actual cutting planes derivation, namely, when the constraint is implied by 
\emph{reverse unit propagation (RUP)}, a generalization of the same notion in SAT~\cite{GN03VerificationProofsUnsatisfiabilityCNFFormulas}. RUP states that if applying integer bounds consistency propagation on $\coreset\cup\derivedset\cup \{\neg C\}$ %
results in a contradiction, then $C$ is implied by $\coreset\cup\derivedset$ and  can hence be derived.

Additionally, \veripb also allows for deriving non-implied constraints as long
as the constraints are guaranteed to preserve satisfiability. Specifically 
\veripb has the following generalization of the resolution asymmetric tautology (RAT) rule in SAT~\cite{JHB12InprocessingRules}. 
\begin{description}
	\item [Redundance-Based Strengthening.]
	The constraint~$C$ can be derived and added to~$\derivedset$ given a  substitution~$\subst$ and explicit (cutting planes)	proofs
for
	\[\coreset  \cup\derivedset\cup \{\neg{C} \}
	\models
	{{(\coreset \cup\derivedset\cup \{C\})}\under{\subst}}
	\cup \order(\varsz\under\subst,\varsz),\]
  i.e., explicit proofs that the  constraint on the right is implied by the premises $\coreset  \cup\derivedset\cup \{\neg{C} \}$.
\end{description}
\noindent Intuitively, this rule ensures that  $\subst$ remaps any solution $\assmnt$ of $\coreset\cup\derivedset$
that does not satisfy~$C$ to a solution $\assmnt\circ\subst$ of $\coreset\cup\derivedset$ that 
(i)~satisfies~$C$ and (ii)~for which $\assmnt\circ\subst\preceq\assmnt$, i.e., $\assmnt\circ\subst$ is least as good (in terms of the order) as $\assmnt$. 
A common application of redundance-based strengthening is
\emph{reification}: deriving 
two pseudo-Boolean constraints that encode
$\lit \lequiv D$
for some PB constraint~$D\in\coreset\cup\derivedset$ and
for some fresh literal $\lit$. 

In addition to adding
constraints, previously derived constraints can also be \emph{deleted} in order to reduce the number of constraints that the proof checker has to work with. However, deletion requires care.
Without restrictions, deleting everything in $\coreset$ could make an unsatisfiable formula satisfiable, which would clearly
be incorrect. Deletion is allowed using the following rules.
\begin{description} 
	\item[Derived Deletion.] Any constraint can be removed from $\derivedset$. 
	\item[Core Deletion.] A constraint $C$ can be removed from $\coreset$ if $C$ can be derived from $\coreset\setminus \{C\}$ with the redundance-based strengthening rule. 
\end{description}

\noindent Solutions found are  logged and excluded from further consideration by learning a constraint using the following rule.
\begin{description}
	\item[Solution Logging.] Given a solution $\assmnt$ to $\coreset\cup\derivedset$, we can derive the constraint $\sum_{\lit \in \assmnt}\olnot\lit \geq 1$ that excludes $\assmnt$ and add this constraint to $\coreset$. 
\end{description}

\noindent Constraints can always be moved from the derived set to the core set using the transfer rule in order to allow e.g.\ the application of core deletion. 
\begin{description}
\item[Transfer Rule.] If $\derivedset'\subseteq\derivedset$, we can transfer from configuration $\langle\coreset,\derivedset,\order,\varsz\rangle$ to  $\langle\coreset\cup(\derivedset\setminus\derivedset'),\derivedset',\order,\varsz\rangle$.
\end{description}  

\noindent Finally, the order can be changed, provided that the derived set is empty (which can always be achieved using the transfer rule).

\begin{description}
	\item[Order Change Rule.] Given a proof in \veripb format that $\order'$ is \emph{reflexive} (i.e., $\order'(\varsu,\varsu)$ is trivial) and transitive (i.e., whenever $\order'(\varsu,\varsv)$ and $\order'(\varsv,\varsw)$ hold, so does $\order'(\varsu,\varsw)$), and given a tuple of variables $\varsz'$ of the right length, we can transition from $\langle\coreset,\emptyset,\order,\varsz\rangle$ to $\langle\coreset,\emptyset,\order',\varsz'\rangle$. 
\end{description}

\section{Proof Logging for Multi-Objective MaxSAT}
\label{sec:main}
As our main contributions,
we will now detail how
\veripb can be used for enabling proof logging in the multi-objective
setting---despite the fact that \veripb does not directly support multiple objectives.
Our solution is based on a new type of use of the preorder $\order$ in \veripb.

\subsection{The General Setup} \label{sec:gen-setup}

Preorders were originally introduced in \veripb to enable proofs for
symmetry breaking~\cite{BGMN23CertifiedDominanceSymmetryBreakingCombinatorialOptimisation}.
However, the preorder turns out to be applicable for multi-objective proof logging as well.
Since all rules in \veripb are guaranteed to preserve solutions that are minimal with respect to the defined preorder, the preorder generalizes a single objective \(\objective\): computing a solution optimal wrt \(\objective\) is equivalent to computing a solution that is smallest in the order \(\order^\objective\) defined by the formula that is true iff \(\objective\under\assmnt \leq \objective\under\assmnttwo\).
As a first step towards the multi-objective setting, we introduce
a suitable order for encoding Pareto-dominance.
In the following definition, if $\varsu$ and $\varsv$ are two tuples of variables of equal length, we write $\subst_{\varsu\to\varsv}$ for the substitution that maps every $u_i$ to $v_i$ and all other variables to themselves. 

\newcommand\orderpareto{\order_P^\objs}
\begin{definition}\label{def:orderpareto}
Let $\objs = (\objective_1,\dots,\objective_p)$ be a tuple of \(p\) objectives over variables $\varsx = (\varx_1,\dots, \varx_k)$, and define 
$\orderpareto(\varsu,\varsv)$ over fresh variables \(\varsu\) and \(\varsv\) as the PB formula
$\orderpareto(\varsu,\varsv) = \{\objective_1\under{\subst_{\varsx\to\varsu}} \leq \objective_1\under{\subst_{\varsx\to\varsv}}, \dots, \objective_p\under{\subst_{\varsx\to\varsu}} \leq \objective_p\under{\subst_{\varsx\to\varsv}}\}$.
\end{definition}
The following proposition summarizes the properties of $\orderpareto$ that are important for our setting.
\begin{proposition} \label{prop:order-property}
Let $\objs = (\objective_1,\dots,\objective_p)$ be a tuple of objectives and $\orderpareto(\varsu,\varsv)$ the PB formula from 
Definition~\ref{def:orderpareto}. Then the following hold:
	\begin{itemize}
		\item $\orderpareto$ encodes a preorder, i.e., a reflexive ($\orderpareto(\varsu,\varsu)$ is  trivially satisfied) and transitive (if $\orderpareto(\varsu,\varsv)$ and $\orderpareto(\varsv,\varsw)$ also $\orderpareto(\varsu,\varsw)$ holds) relation. 
		\item \(\orderpareto(\varsx\under\assmnt,\varsx\under\assmnttwo)\) is satisfied if and only if \(\assmnt\weaklyDominates\assmnttwo\) wrt $\objs$, i.e., if \(\assmnt\) weakly dominates \(\assmnttwo\). 
	\end{itemize}
\end{proposition} 
\renewcommand\orderpareto{\order_P}
When the objectives are clear from context, we drop the superscript and use $\orderpareto(\varsu,\varsv)$ for the order that encodes 
Pareto-dominance over the objectives.

With the transfer and order change rules, the order can be changed arbitrarily in \veripb proofs.
In our setting, however, we will use $\orderpareto$ unchanged throughout the entire proof. From
now on, a \emph{\veripb proof for $(\formula,\objs)$} refers to
 a standard \veripb proof for $\formula$ that
(i)~as the first derivation step loads the order $\orderpareto$ over the variables \(\varsx\) in the objectives \(\objs\), and
(ii)~at no other point in the proof changes the order. 
Our observations on  valid \veripb proofs assume that the conditions (i) and~(ii) are satisfied.
The \veripb proof checker will not verify these two conditions for us; however, these are merely syntactic restrictions that can be verified easily (\eg by checking that no other lines in the proof starts with \texttt{load\_order}).

The following result now guarantees the correctness of the proofs produced for the different MO-MaxSAT algorithms. 

\begin{theorem}\label{thm:main}
	Let $\proofVeriPB$ be a \veripb proof for $(\formula,\objs)$ that derives a contradiction. 
	Let $S$ be the set of non-dominated solutions logged in $\proofVeriPB$, i.e., logged solutions that are not dominated by other solutions logged in $\proofVeriPB$.
  Then $S$ contains a representative solution for each element in the non-dominated set of $(\formula,\objs)$. 
\end{theorem}

Let \(\langle\coreset_i,\derivedset_i,\order_i,\varsz_i\rangle\) be the proof configuration of \(\proofVeriPB\) at step \(i\).
From the additional conditions imposed for multi-objective proofs, we have
 \(\order_i = \orderpareto\) and \(\varsz_i = \varsx\) for all \(i \geq 1\), since
\(\orderpareto\) over the variables \(\varsx\) in \(\objs\) is loaded in the
first derivation step.

Theorem~\ref{thm:main} follows by the following lemmas.
Lemma~\ref{lem:derivedset} is a restatement of the properties of \veripb proofs shown in~\cite{BGMN23CertifiedDominanceSymmetryBreakingCombinatorialOptimisation}, included here for completeness.
\begin{lemma}\label{lem:derivedset}
Let $\langle\coreset_i,\derivedset_i,\orderpareto,\varsx\rangle$ be the \nth{i} configuration of $\proofVeriPB$.
For every solution $\assmnt$ of $\coreset_i$, there exists a solution  $\assmnttwo$ of $\coreset_i \cup \derivedset_i$ for which 
$\orderpareto(\varsx\under{\assmnttwo}, \varsx\under\assmnt)$.
\end{lemma} 
The next lemma establishes that no rule in $\veripb$ can ``create'' new non-dominated points.  
\begin{lemma}\label{lem:no-domination}
  Consider the \nth{i} configuration $\langle\coreset_i,\derivedset_i,\orderpareto,\varsx\rangle$ of $\proofVeriPB$ for fixed $i \geq 1$.
Any solution of $\coreset_i$ %
is weakly dominated by a Pareto-optimal solution of $(\formula,\objs)$.
\end{lemma}
\begin{proof}
By induction on $i$. The base case $i=1$ follows by Proposition~\ref{prop:order-property} from the first configuration having $\coreset_1 = \formula$. Assume that the statement holds for $i-1$ and let $\langle\coreset_{i-1},\derivedset_{i-1},\orderpareto,\varsx\rangle$ be the \nth{i{-}1} configuration. The rules of $\veripb$ that can alter the core set are solution logging, core deletion, and the transfer rule. For the transfer rule, the statement follows immediately by Lemma~\ref{lem:derivedset}.
For solution logging, the result follows from any solution of \(\coreset_{i}\) being a solution of \(\coreset_{i-1}\). 
Assume thus that $\coreset_i = \coreset_{i-1} \setminus \{\constrc\}$ and
let $\assmnt$ be a 
solution of $\coreset_i$, that (for the non-trivial case) does not satisfy \(\constrc\). We show that $\assmnt$ is
weakly dominated by a Pareto-optimal solution of $(\formula, \objs)$. 
By the properties of redundance-based strengthening and core deletion, there is a substitution $\subst$ such that $\assmnttwo = \assmnt \circ \subst$ 
is a solution to $\coreset_{i-1}$ for which $\orderpareto(\varsx\under{\assmnttwo}, \varsx\under\assmnt)$. By the induction assumption, $\assmnttwo$ is weakly dominated by a Pareto-optimal solution \(\gamma\) of $(\formula, \objs)$. Since \(\assmnttwo\weaklyDominates\assmnt\), $\assmnt$ is  weakly dominated by \(\gamma\).
\qed
\end{proof}

Lemma~\ref{lem:atleastone} establishes that no rule except for solution logging can remove all representative solutions for an element in the non-dominated set of $(\formula, \objs)$. 
\begin{lemma}\label{lem:atleastone}
Let $\langle\coreset_{i-1},\derivedset_{i-1},\orderpareto,\varsx\rangle$ and $\langle\coreset_i,\derivedset_i,\orderpareto,\varsx\rangle$
be the \nth{i{-}1} and \nth{i} configurations of $\proofVeriPB$, respectively. Assume that (i)~the solutions of $\coreset_{i-1} \cup \derivedset_{i-1}$ include a representative for the non-dominated point $\mathbf{b}$ and (ii)~the \nth{i} configuration is obtained by a rule other than solution logging. 
Then the solutions of $\coreset_{i} \cup \derivedset_{i}$ include a representative solution for $\mathbf{b}$.
\end{lemma}
\begin{proof}
Let $\assmnt$ be the representative solution to $\coreset_{i-1} \cup \derivedset_{i-1}$ for $\mathbf{b}$. We construct a solution to $\coreset_i \cup \derivedset_i$ that is 
representative for $\mathbf{b}$. For the interesting case, assume that $\assmnt$ is not a solution to $\coreset_i \cup \derivedset_i$. By assumption the
\nth{i} configuration was reached by a rule other than solution logging, and hence
$\derivedset_i = \derivedset_{i-1} \cup \{\constrc\}$  for a constraint \(\constrc\) added using the redundance-based strengthening rule and a substitution $\subst$. 
By redundance-based strengthening $\assmnttwo =  \assmnt \circ \subst$ is a solution to $\coreset_i \cup \derivedset_i$ for which 
$\orderpareto(\varsx\under{\assmnttwo}, \varsx\under\assmnt)$. Hence by Proposition~\ref{prop:order-property}, 
$\assmnttwo$ weakly dominates $\assmnt$.
As \(\assmnt\) is representative for the non-dominated point \(\mathbf{b}\), \(\assmnttwo\) must be
representative for \(\mathbf{b}\) as well.
\qed
\end{proof}

Using these three lemmas we can now establish Theorem~\ref{thm:main} as follows.
\begin{proof}[of Theorem~\ref{thm:main}]
Consider the first configuration
$\langle\coreset_1,\derivedset_1,\orderpareto,\varsx\rangle$ of $\proofVeriPB$, where \(\coreset_1 = \formula\) and \(\derivedset_1 = \emptyset\). By
Proposition~\ref{prop:order-property}, there is a one-to-one correspondence
between the Pareto-optimal solutions of $(\formula,\objs)$ and the solutions of
$\coreset_1 \cup \derivedset_1$ that are minimal wrt $\orderpareto$.
Specifically, the solutions of $\coreset_1 \cup \derivedset_1$ contain a
representative solution for each element in the non-dominated set of
$(\formula,\objs)$. Since $\proofVeriPB$ derives a contradiction, there are no
solutions to the union of the core and derived set of the final configuration.
Thus the theorem holds if
(i)~no solutions
that dominate a Pareto-optimal solution of $(\formula, \objs)$ are logged in
the proof, and (ii)~a representative solution for each element in the
non-dominated set is logged in the proof. (i)~follows by
Lemma~\ref{lem:no-domination} and (ii)~by Lemma~\ref{lem:atleastone}.
\qed
\end{proof}

We note that
Theorem~\ref{thm:main} also 
holds when an instance $(\formula, (\objective))$ only has a single objective.
In this case $\assmntone\weaklyDominates\assmnttwo$ is equivalent to $\objective \under \assmntone \leq 
\objective \under \assmnttwo $.
Hence Theorem~\ref{thm:main} states that
the simplified version of \veripb using only rules for decision problems (recall
Section~\ref{sec:proofsystem}) is sufficient for certifying single-objective MaxSAT solvers
as well. Thereby it can be argued that the (more complicated) optimization proof 
system presented  in~\cite{BGMN23CertifiedDominanceSymmetryBreakingCombinatorialOptimisation}
could be simplified without losing any of its expressiveness by replacing  
the explicit linear objective and objective-specific rules by the 
linear order-based rules we describe.

\newcommand\witness{\omega}

\subsection{Proof Logging for Pareto Dominance Cuts}\label{sec:proof-pdcuts}

A key step in 
 proof logging the different multi-objective algorithms is the derivation of a constraint called a \pdcut given a solution $\assmnt$. 
In particular,
 given a solution $\assmnt$ of $\formula$, we will derive a constraint that states that we are no longer interested in solutions worse 
than or equally good as $\alpha$ in terms of Pareto-dominance. 
In single-objective \veripb, there is a dedicated rule that allows for deriving
a so-called solution-improving constraint. However
we will show that \pdcuts can be derived only relying on the redundance-based strengthening and solution logging rules. 

We will make use of some auxiliary variables \(\worse_i\) for each objective \(\objective_i\) (recall Equation~\eqref{eq:worsedef}).
Firstly, introducing such (reified) constraints can be done with redundance-based strengthening in a standard way~\cite{GN21CertifyingParityReasoningEfficientlyUsingPseudo-Boolean}.  
The introduced constraints guarantee that iff \(\assmnttwo\) is worse than or equal to \(\assmnt\) in \(\objective_i\), then variable $\worse_i$ holds in $\assmnttwo$.
\newcommand\witnessalpha{{\subst_\alpha}}
Now let $\witnessalpha$ be the substitution that maps every variable to their value in $\alpha$ and each \(\worse_i\) to $1$.
We claim that with the redundance-based strengthening rule and this witness we can derive the constraint 
\[ \constrc_\alpha := \sum_{i=1}^p |\alpha|\cdot \olnot{\worse_i}   + \sum_{\lit \in \alpha}\lit \geq |\alpha|. \]
Intuitively, this constraint maps each solution weakly dominated by \(\assmnt\) to \(\assmnt\).
In doing so it excludes all solutions the \pdcut excludes (the solutions weakly dominated by \(\assmnt\)) except for $\assmnt$ itself.
To see that this constraint can
be derived with redundance-based strengthening, we 
can verify that all proof obligations are indeed met.
\begin{itemize}
  \item For each constraint $\constrc$ in $\coreset\cup\derivedset$  that does not mention $\worse_i$, $\constrc\under\witnessalpha$ is trivially satisfied since $\alpha$ is a solution that satisfies those constraints. 
\item If $\constrc$ is one of the constraints \eqref{eq:worsedef}, then clearly $\constrc\under\witnessalpha$ holds too. 
  \item Clearly also $\constrc_\alpha\under\witnessalpha$ holds. 
  \item Lastly, what we need to show is that $\order(\varsx\under\witnessalpha,\varsx)$ holds if \(\constrc_\alpha\) is not satisfied. This constraint expresses that $\alpha$ weakly dominates
    any assignment $\beta$ that satisfies all derived constraints so far but not $\constrc_\alpha$. 
    As such an assignment $\beta$ must assign all $\worse_i$ to $1$,
    from \eqref{eq:worsedef} we immediately have
that $\objective_i\under\beta\geq\objective_i\under\alpha$, as desired.
\end{itemize}

Finally, after deriving $\constrc_\alpha$, we log the solution $\alpha$ to obtain a solution-excluding constraint
  $\sum_{\lit\in\alpha}\olnot\lit \geq 1$.
By adding $\sum_{\lit\in\alpha}\olnot\lit \geq 1$ to $\constrc_\alpha$, we arrive at
$\sum_{i=1}^p |\alpha|\cdot \olnot{\worse_i}   \geq 1$.
Hence  at least one of the $\worse_i$ must be $0$. Dividing the result by $|\alpha|$ yields the \pdcut \((\olnot{\worse_1} \lor \dots \lor \olnot{\worse_p})\).

\newcommand\pid[1]{[#1]}
\begin{example}
  Recall the instance in Fig.~\ref{fig:running-ex}.
  Table~\ref{tab:pd-cut-proof} shows the steps required in the \veripb proof
  for certifying a \pdcut based on solution \(\assmnt =
  \{\varx_1,\olnot{\varx_2},\varx_3,\olnot{\varx_4},\olnot{\varx_5}\}\) with
  objective values \((4,8)\).
  Steps \(\pid{a}\), \(\pid{b}\), \(\pid{c}\), and \(\pid{d}\) first introduce the definitions of the
  $\worse^\assmnt_1$ and $\worse^\assmnt_2$ variables from Equation~\eqref{eq:worsedef} as normalized reified PB constraints.
  These steps are justified by redundance-based strengthening using the fresh variables
  \(\worse^\assmnt_1\) and \(\worse^\assmnt_2\) as witnesses.
  Next, step \(\pid{e}\) introduces \(\constrc_\assmnt\).
  Lastly, the solution \(\assmnt\) is logged (in \(\pid{f}\)) and the \pdcut derived (in
  \(\pid{g}\)).
  \begin{table}[t]
    \caption{Example proof for certifying a \pdcut.}\label{tab:pd-cut-proof}
    \centering
    \begin{tabular}{@{}r@{\hspace{1em}}r@{\hspace{1em}}c@{\hspace{1em}}c@{}}
      \toprule
      ID & \multicolumn{1}{c}{Pseudo-Boolean Constraint} & Comment & Justification \\
      \midrule
      \multicolumn{4}{c}{\emph{Input constraints and potential previous proof steps}} \\
      \(\pid{a}\) & \(11\worse^\assmnt_1 + 3 \olnot{\varx_2} + 4 \olnot{\varx_3} + 2 \olnot{\varx_4} + 5 \olnot{\varx_5} \ge 11\) & \(\worse^\assmnt_1\limpliedby\objective_1 \ge \objective_1\under\assmnt\) & Redundance \(\{\worse^\assmnt_1\}\) \\
      \(\pid{b}\) & \(4\olnot{\worse^\assmnt_1} + 3 \varx_2 + 4 \varx_3 + 2 \varx_4 + 5 \varx_5 \ge 4\) & \(\worse^\assmnt_1 \limplies \objective_1 \ge \objective_1 \under \assmnt\) & Redundance \(\{\olnot{\worse^\assmnt_1}\}\) \\

      \(\pid{c}\) & \(7\worse^\assmnt_2 + 7 \olnot{\varx_1} + 4 \olnot{\varx_2} + 1 \olnot{\varx_3} + 2 \olnot{\varx_4} \ge 7\) & \(\worse^\assmnt_2\limpliedby\objective_2 \ge \objective_2\under\assmnt\) & Redundance \(\{\worse^\assmnt_2\}\) \\
      \(\pid{d}\) & \(8\olnot{\worse^\assmnt_2} + 7 \varx_1 + 4 \varx_2 + 1 \varx_3 + 2 \varx_4 \ge 8\) & \(\worse^\assmnt_2 \limplies \objective_2 \ge \objective_2 \under \assmnt\) & Redundance \(\{\olnot{\worse^\assmnt_2}\}\) \\

      \(\pid{e}\) & \(5 \olnot{\worse^\assmnt_1} + 5 \olnot{\worse^\assmnt_2} + \varx_1 + \olnot{\varx_2} + \varx_3 + \olnot{\varx_4} + \olnot{\varx_5} \ge 5\) & \(\constrc_\assmnt\) & Redundance \(\witnessalpha\) \\
      \(\pid{f}\) & \(\olnot{\varx_1} + \varx_2 + \olnot{\varx_3} + \varx_4 + \varx_5 \ge 1\) & & Log solution \(\assmnt\) \\
      \(\pid{g}\) & \(\olnot{\worse^\assmnt_1} + \olnot{\worse^\assmnt_2} \ge 1 \) & \pdcut & \((\pid{e} + \pid{f}) / 5\) \\
      \bottomrule
    \end{tabular}
  \end{table}
\end{example}

\subsection{Proof Logging Multi-Objective MaxSAT Algorithms} 

To enforce bounds on the values of the different objectives, multi-objective MaxSAT
algorithms make use of CNF encodings of (reified) pseudo-Boolean constraints.
For certifying the correctness of MO-MaxSAT algorithms, the correctness of
these encodings needs to be certified as well. 
All algorithms covered in this paper make use of the 
incremental
(generalized) totalizer
encoding~\cite{BailleuxBoufkhad2003EfficientCNFEncoding,JoshiEtAl2015GeneralizedTotalizerEncoding,MartinsEtAl2014IncrementalCardinalityConstraints}. 
The totalizer encoding can be visualized as a tree where
each node has a set of output literals that ``count'' how many of the 
literals at the leaves of its subtree are \(1\).
In earlier work~\cite{VDB22QMaxSATpbCertifiedMaxSATSolver} it has been shown how the clauses of the totalizer encoding can be 
derived in the \veripb proof system from the constraints describing the
semantics of the output and internal variables. For certifying the
\emph{generalized} totalizer encoding, the semantics of the output
variables are slightly different. Each output variable \(o_b\) of a node is
defined by the two constraints \( o_b \lequiv \sum_{i=1}^n a_i\lit_i \geq b \)
over the \(n\) literals at the leaves of the subtree. In contrast to the
unweighted case, \(a_i\) can be larger than $1$ here and values of \(b\) that are
not subsetsums of \(\{a_i \mid i=1,\dots,n\}\) are omitted. Deriving the
clauses of the generalized totalizer encoding now follows the cutting
planes procedure described in~\cite{VDB22QMaxSATpbCertifiedMaxSATSolver}.

Many implementations of MaxSAT algorithms employing the totalizer encoding only
derive clauses for enforcing the \( o_b \limpliedby \sum_{i=1}^n a_i\lit_i \geq b
\) constraint, since these are enough for enforcing upper bounds on the
objective values by setting \(o_b\) to \(0\). As observed in~\cite{JabsEtAl2024FromSingleObjective}, for
the generalized totalizer it is additionally necessary to enforce \emph{all}
output variables in the range \([b,\max\{a_i \mid i=1,\dots,n\})\) to \(0\) in
order to enforce \(\sum_{i=1}^n a_i\lit_i < b\). With this modification, also for
the generalized totalizer encoding, deriving only one ``direction'' of clauses
is enough from the perspective of the solver.
In the proof, however, we will make use of both directions of the definition of these output variables; this means that a solution found by the SAT solver is not
necessarily a solution to the constraints in the proof. When using a
solution found by the SAT solver as a witness for the redundance-based strengthening rule for deriving a \pdcut (recall Section~\ref{sec:proof-pdcuts}), we
therefore need to adjust the assignment to satisfy the stricter semantics of
the proof first. This is done during proof generation by traversing through all
nodes of the (generalized) totalizer encodings and manually assigning the
output variables to values following the strict semantics described
above. This adjusted assignment is still guaranteed to
satisfy all clauses in the SAT solver.

Next, we detail three state-of-the-art MO-MaxSAT algorithms and how to generate proofs for them.

\paragraph{\pminimal~\citepmin.}
Starting from any solution \(\assmnt\), the \pminimal algorithm introduces a \pdcut excluding
all solutions that are weakly dominated by \(\assmnt\). A SAT solver is then queried while
temporarily enforcing the next-found solution to dominate
\(\assmnt\).
Both steps are achieved using 
the generalized totalizer encoding.
If no solution dominating the latest one can be found, the previous solution is
guaranteed to be Pareto-optimal.
In this case \pminimal drops the temporary constraints and starts over.
If the working formula is unsatisfiable at this point, \pminimal terminates.
An example of a search path of \pminimal in objective space is illustrated on
the left-hand side of Fig.~\ref{fig:traces}.
The blue circles and red dots represent the solutions found by the SAT solver,
with the red dots representing Pareto-optimal solutions.
In such an execution, \pminimal introduces a \pdcut for each of these solutions.

\begin{figure}[t]
  \centering
  \begin{tabular}{ccc}
    \pminimal & \lowerbounding & \bioptsat \\
    \includegraphics{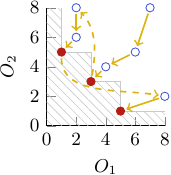} &
    \includegraphics{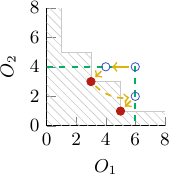} &
    \includegraphics{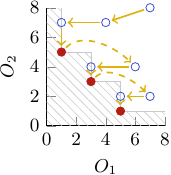} \\
  \end{tabular}
  \caption{Illustrations of the search path of MO-MaxSAT algorithms in objective space.}\label{fig:traces}
\end{figure}

For proof logging  \pminimal we certify the generalized totalizer
objective encodings and the added \pdcuts in the \veripb proof (as already described).
This allows a contradiction to be derived in the proof iff a \pdcut was added
for each element in the non-dominated set.
The temporary constraints that enforce domination are not required in the proof since
they are merely heuristics guiding the SAT solver to certain regions temporarily
and are not needed for reaching the final contradiction. 

\paragraph{\lowerbounding~\citelb.}
The \lowerbounding algorithm restricts the search space by temporarily
enforcing upper bounds of the form \(\objective_i \leq b_i\) on each objective
and then executes \pminimal within these bounds. Once \pminimal terminates, the
bounds \(b_i\) are loosened and the process is repeated. \lowerbounding
terminates once the bounds include the entire search space. In this last case,
\pminimal is executed without the temporary bound constraints and will
terminate with an unsatisfiable working instance. For proof logging
\lowerbounding it therefore suffices to proof log the invocations of \pminimal
as a subroutine.
The search path for one invocation of \pminimal within a set of bounds is illustrated in green in Fig.~\ref{fig:traces}~(middle).

\paragraph{\bioptsat~\citebos.}
The \bioptsat algorithm is specific to problems with two objectives
\((\objective_1, \objective_2)\).
The algorithm  enumerates non-dominated points under
the guarantee that the values for one objective are monotonically increasing,
while the values of the other objective are decreasing.
\bioptsat first employs the subroutine \mininc to minimize \(\objective_1\) without any
additional constraints, returning the solution \(\assmnt\).
Next, the subroutine \mindec uses solution-improving search to minimize
\(\objective_2\) under the condition that \(\objective_1 =
\objective_1\under\assmnt\).
Let the final solution found by \mindec be \(\assmnttwo\), which is guaranteed
to be Pareto-optimal.
\bioptsat then repeats this process after introducing the constraint
\(\objective_2 \leq \objective_2\under\assmnttwo - 1\).
An example of a search path of \bioptsat is illustrated in
 Fig.~\ref{fig:traces} (right).
Since \mininc finds the global minimum of \(\objective_1\) for the current
working formula, assume for now that we have a proof logging procedure for
\mininc that results in the constraint \(\constrc_\mininc := \objective_1 \ge
\objective_1\under\assmnt\) being added to the proof.
For any solution found during the execution of \mindec, we introduce a
\pdcut, which in turn is strengthened to the unit clause \((\olnot
{\worse_2})\) by combining it with \(\constrc_\mininc\).
This strengthened \pdcut is semantically equivalent to the constraint added
during solution-improving search in \mindec and at the end before \bioptsat
starts over.

\begin{example}
  Recall again the  instance in Fig.~\ref{fig:running-ex}.
  Assume that we have executed \mininc and \mindec once already and that
  \mindec has yielded the Pareto-optimal \(\assmnt =
  \{\varx_1,\olnot{\varx_2},\varx_3,\olnot{\varx_4},\olnot{\varx_5}\}\).
  Table~\ref{tab:bos-proof} details the proof steps taken to certify the
  constraint \(\objective_2 < \objective_2\under\assmnt\) added as the last
  step in \bioptsat.
  First, assume that \(\pid{lb}\) is \(\constrc_\mininc\) derived in the proof
  during \mininc.
  Next, we certify the \pdcut for \(\assmnt\) in the proof as detailed in
  Table~\ref{tab:pd-cut-proof}.
  Finally, by summing up the lower-bounding constraint from \mininc, the
  definition of \(\worse^\assmnt_1\) and the \pdcut, we obtain
  \(\olnot{\worse^\assmnt_2} \ge 1\) which is semantically equivalent to the
  constraint added by \bioptsat.
  \begin{table}[t]
    \caption{Example proof for a strengthened \pdcut in \bioptsat.}\label{tab:bos-proof}
    \centering
    \begin{tabular}{@{}r@{\hspace{1em}}r@{\hspace{1em}}c@{\hspace{1em}}c@{}}
      \toprule
      ID & \multicolumn{1}{c}{Pseudo-Boolean Constraint} & Comment & Justification \\
      \midrule
      \multicolumn{4}{c}{\emph{Input constraint and potential previous proof steps}} \\
      \(\pid{lb}\) & \(3 \varx_2 + 4 \varx_3 + 2 \varx_4 + 5 \varx_5 \ge 4\) & & Derived during \mininc \\
      \(\pid{a}\) & \(11\worse^\assmnt_1 + 3 \olnot{\varx_2} + 4 \olnot{\varx_3} + 2 \olnot{\varx_4} + 5 \olnot{\varx_5} \ge 11\) & \(\worse^\assmnt_1 \limpliedby \objective_1 \ge \objective_1\under\assmnt\) & Redundance \(\{\worse^\assmnt_1\}\) \\
      \multicolumn{4}{c}{\emph{Other \pdcut certification steps from Table~\ref{tab:pd-cut-proof}}} \\
      \(\pid{g}\) & \(\olnot{\worse^\assmnt_1} + \olnot{\worse^\assmnt_2} \ge 1\) & \pdcut & See Table~\ref{tab:pd-cut-proof} \\
      \(\pid{h}\) & \(\worse^\assmnt_2 \ge 1\) & & \(((\pid{lb} + \pid{a}) / 11) + \pid{g}\) \\
      \bottomrule
    \end{tabular}
  \end{table}
\end{example}

It remains to show that we can proof log \mininc to derive \(\constrc_\mininc\).
The details depend on the implementation of \mininc for which several
variants have been proposed~\cite{JabsEtAl2024FromSingleObjective}.
For the SAT-UNSAT variant, where \mininc is implemented as solution-improving
search, \(\constrc_\mininc\) is implicitly derived by the SAT solver during the
final unsatisfiable query.
Proof logging for the core-guided variants of \mininc can be implemented as
already described for the single-objective setting
in~\cite{BBNOV23CertifiedCore-GuidedMaxSATSolving}; in short,
\(\constrc_\mininc\) can be derived from the final reformulated objective. %

\subsection{Proof Logging Core Boosting}

Core boosting~\cite{JBJ24CoreBoostingSAT-BasedMulti-objectiveOptimization} is a recently-proposed
preprocessing/reformulation technique for MO-MaxSAT, consisting of applying the single-objective core-guided
optimization algorithm OLL~\cite{AKMS12Unsatisfiability-basedoptimizationclasp,DBLP:conf/cp/MorgadoDM14} wrt each objective
individually, before executing an MO-MaxSAT algorithm on the reformulated objectives obtained from OLL.
With this, core boosting shrinks the search space that needs to be considered by
the MO-MaxSAT algorithm by deriving lower bounds for each objective.
Core boosting also alters the structure of the CNF objective encodings since the totalizer
structures built by OLL during core boosting can be reused in the objective
encodings built by the MO-MaxSAT algorithm.

We give a brief overview of the single-objective core-guided OLL MaxSAT
algorithm to the extent relevant for understanding how proof logging for core
boosting works.
Given an objective \(\objective\), OLL invokes a SAT solver with the
assumptions that none of the literals in \(\objective\) incurs cost.
If these assumptions are not satisfiable, the SAT solver returns an implied clause
\(\constrc\)---referred to as an unsatisfiable core---over the objective literals.
OLL now introduces counting variables \(o_i \limpliedby \sum_{\lit \in \constrc} \lit
\ge i\) for \(i=2,\dots,|\constrc|\) (encoded by the totalizer encoding) and
reformulates the objective by adding \(c_\constrc \cdot (\sum_{\lit \in \constrc} -\lit
+ \sum_{i=2}^{|\constrc|} o_i + 1)\) to it, where \(c_\constrc\) is the minimum
objective coefficient of any literals in \(\constrc\).
Iteratively applying this process guarantees that the reformulated objective is
always equal to the original objective and OLL terminates once there is a
solution that does not incur cost on any of the literals in the reformulated
objective.

\begin{example}\label{ex:oll}
  Let \(\objective = \varx_1 + \varx_2 + \varx_3 + 2 \varx_4\) be one of the
  objectives that core boosting is performed on.
  Assume the first core extracted is \(\constrc_1 = (\varx_1 \lor \varx_2 \lor
  \varx_3 \lor \varx_4)\).
  OLL reformulates \(\constrc_1\) by adding counting variables \(a_i
  \limpliedby \varx_1 + \varx_2 + \varx_3 + \varx_4 \ge i\) for \(i=2,3,4\).
  Assume the next core extracted is \(\constrc_2 = (\varx_4 \lor a_2)\), which is
  reformulated by adding the counting variable \(b_2 \limpliedby \varx_4 + a_2
  \ge 2\).
  The final reformulated objective is \(\objective' = a_3 + a_4 + b_2 + 2\).
\end{example}

Proof logging for OLL is described
in~\cite{BBNOV23CertifiedCore-GuidedMaxSATSolving}, yielding a constraint of
form \(\objective \ge \objective'\) associating the reformulated objective
\(\objective'\) with the original objective \(\objective\).
After OLL has been executed, core boosting builds a CNF encoding for each
reformulated objective.
However, if the reformulated objective contains a sequence of literals that are
outputs for the same totalizer, they should be reused as an internal node of
the final encoding employed by the MO-MaxSAT algorithm rather than individual leaves.
This avoids introducing new auxiliary variables that would end up being
equivalent to the variables introduced by OLL.

\begin{example}
  Fig.~\ref{fig:cb-enc} (left and middle-left, respectively) shows the two totalizer
  encodings built by OLL, where \(\square\) denotes the output variable with
  value 1 that is omitted by OLL.
  Fig.~\ref{fig:cb-enc} (middle-right) shows the encoding of the
  reformulated objective built by core boosting after executing OLL, where
  \(a_3,a_4\) is reused as an internal node.
  The dashed box shows the encoding that would be built when treating \(a_3\)
  and \(a_4\) individually.
  Since \(a_3\) and \(a_4\) are already totalizer outputs, \(d_1\) and \(d_2\)
  in this structure are equivalent to \(a_3\) and \(a_4\) and therefore
  redundant.

  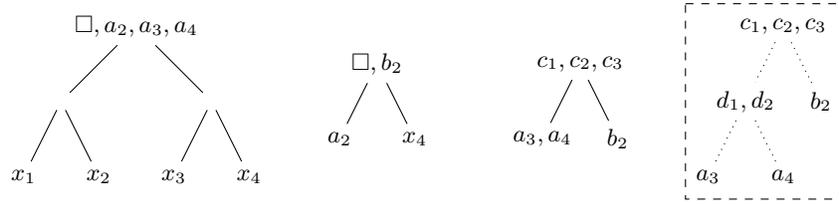
\begin{figure}[t]
    \begin{center}
      \begin{tikzpicture}
        \node (a) at (1.5,0) {\(\square,a_2,a_3,a_4\)};
        \node (i1) at (0.5,-1) {};
        \node (i2) at (2.5,-1) {};
        \node (x1) at (0,-2) {\(x_1\)};
        \node (x2) at (1,-2) {\(x_2\)};
        \node (x3) at (2,-2) {\(x_3\)};
        \node (x4) at (3,-2) {\(x_4\)};

        \draw (a) -- (i1);
        \draw (a) -- (i2);
        \draw (i1) -- (x1);
        \draw (i1) -- (x2);
        \draw (i2) -- (x3);
        \draw (i2) -- (x4);

        \begin{scope}[shift={(4.2,-.5)}]
          \node (b) at (.5,0) {\(\square,b_2\)};
          \node (a2) at (0,-1) {\(a_2\)};
          \node (x5) at (1,-1) {\(x_4\)};

          \draw (b) -- (a2);
          \draw (b) -- (x5);
        \end{scope}

        \begin{scope}[shift={(6.9,-.5)}]
          \node (c) at (.5,0) {\(c_1,c_2,c_3\)};
          \node (a34) at (0,-1) {\(a_3,a_4\)};
          \node (b2) at (1,-1) {\(b_2\)};

          \draw (c) -- (a34);
          \draw (c) -- (b2);
        \end{scope}

        \begin{scope}[shift={(9.1,0)}]
          \node (c) at (1,0) {\(c_1,c_2,c_3\)};
          \node (d12) at (.5,-1) {\(d_1,d_2\)};
          \node (b2) at (1.5,-1) {\(b_2\)};

          \draw [dotted] (c) -- (d12);
          \draw [dotted] (c) -- (b2);

          \node (a3) at (0,-2) {\(a_3\)};
          \node (a4) at (1,-2) {\(a_4\)};

          \draw [dotted] (d12) -- (a3);
          \draw [dotted] (d12) -- (a4);

          \draw [dashed] (-.3,.3) -- (1.8,.3) -- (1.8,-2.3) -- (-.3,-2.3) -- cycle;
        \end{scope}

      \end{tikzpicture}
    \end{center}
    \caption{An objective encoding structure built by core boosting. 
    The dashed box shows the alternative structure that would
    be built without reusing \(a_3,a_4\) as an internal node.}\label{fig:cb-enc}
  \end{figure}
\end{example}

\newcommand\constrordabove{\constrc^{a}}
\newcommand\constrordbelow{\constrc^{b}}
Going beyond previous work, certifying the encoding of the reformulated
objective built during core boosting requires special care.
In particular, due to the process of reusing partial sequences of totalizer output variables as internal nodes, using the semantics of the totalizer outputs built by OLL does
\emph{not} allow us to derive the clauses required for the encoding.
Instead, for a sequence of variables \(o_r,\dots,o_{r+n}\)  reused as
an internal node in the encoding, we introduce ordering constraints
\[ \constrordabove_v := \sum_{i=r}^{v} [\olnot{o_i} \ge 0] + \sum_{i=v+1}^{r+n} [\olnot{o_i} + o_v \ge 1]
= \left[ M o_v + \sum_{i=r}^{r+n} \olnot{o_i} \ge M \right]
, \]
where \(M = r+n-v+1\), and
\[ \constrordbelow_v := \sum_{i=v+1}^{r+n} [o_i \ge 0] + \sum_{i=r}^{v-1} [\olnot{o_v} + o_i \ge 1]
= \left[ (v-r)\olnot{o_v} + \sum_{i=r}^{r+n} o_i \ge (v-r) \right]
, \]
for each \(v=r,\dots,r+n\).
These constraints sum up axioms and ordering constraints, which can be derived
from the semantic definitions of the totalizer output variables, and are
therefore derivable in the proof.
Furthermore, \(\constrordabove_v,\constrordbelow_v\) are identical to \(o_v
\lequiv \sum_{i=r}^{r+n} o_i \ge (v-r)\).
When deriving the clauses involving the reused output variable sequence, we
therefore use \(\constrordabove_v\) and \(\constrordbelow_v\) instead of the
semantic definitions of the variables, which allows for deriving the clauses of
the encoding.

After core boosting, the output variables of the objective encoding are now
defined with respect to the \emph{reformulated} objective rather than the
original one.
As a final step, after certifying a \pdcut with respect to the original
objectives, we therefore use the objective reformulation constraints derived
while executing OLL to certify the \pdcut with respect to the reformulated
objectives.

\begin{example}
  All required clauses for the  encoding shown in
  Fig.~\ref{fig:cb-enc} (middle-right) can be derived from the semantic definitions \(c_i
  \lequiv a_3 + a_4 + b_2 \ge i\) for \(i=1,2,3\) while treating the individual
  variable \(b_2\) as a leaf and using \(\constrordabove\) and
  \(\constrordbelow\) as ``pseudo semantics'' for the node \(a_3,a_4\).
  Note that the semantics for \(c_i\) are with respect to the variable part of
  the reformulated objective.
\end{example}

\section{Experiments}

We extended all algorithms implemented in the
Scuttle~\cite{scuttle,JBIJ23PreprocessingSAT-BasedMulti-ObjectiveCombinatorialOptimization,JBJ24CoreBoostingSAT-BasedMulti-objectiveOptimization}
MO-MaxSAT solver---namely,
\pminimal~\citepmin, \lowerbounding~\citelb, and (the SAT-UNSAT variant of) \bioptsat~\citebos,
each with and without core boosting---with the just-described \veripb proof logging.
We used CaDiCaL 2.0.0~\cite{BFFFFP24CaDiCaL20} as the SAT solver within Scuttle and
the \veripb 2.2.2 proof checker~\cite{veripbchecker}
for checking the produced proofs.\footnote{Both \veripb and CaDiCaL contain bug-fixes obtained
directly from their respective authors based on our reporting while preparing our
experiments.}
The proof logging Scuttle implementation is available in open source~\cite{scuttle,jabs_2025_14731485}.
We evaluate the implementation on the same set of benchmark instances
used in the original work proposing core boosting~\cite{JBJ24CoreBoostingSAT-BasedMulti-objectiveOptimization}.
This set of benchmarks consists of 300 instances from 6 domains with the number
of objectives ranging from 2 to 5.
The experiments were run on 2.50-GHz Intel Xeon Gold 6248 machines with 381-GB RAM in RHEL
under a 32-GB memory limit and 1 hour time limit for Scuttle.

\begin{table}[t]
  \centering
  \caption{Average proof logging  overheads and average proof checking overheads.\label{tab:overheads}}
  \begin{tabular}{@{}lrrrr@{}}
    \toprule
    & \multicolumn{2}{c}{with core boosting} & \multicolumn{2}{c}{without core boosting} \\
    Algorithm & \(\frac{\text{solving w/proof log}}{\text{solving}}\) & \(\frac{\text{proof checking}}{\text{solving w/proof log}}\) & \(\frac{\text{solving w/proof log}}{\text{solving}}\) & \(\frac{\text{proof checking}}{\text{solving w/proof log}}\) \\
    \midrule
    \pmin          & \(1.233\) & \(47.52\) & \(1.176\) & \(47.89\) \\
    \bos           & \(1.247\) & \(29.70\) & \(1.140\) & \(18.70\) \\
    \lowerbounding & \(1.220\) & \(21.81\) & \(1.294\) & \(21.78\) \\
    \bottomrule
  \end{tabular}
\end{table}

The per-instance
Scuttle runtimes for each benchmark domain with and without proof logging are
shown in Fig.~\ref{fig:logging-overhead}
for all three MO-MaxSAT algorithms in Scuttle, both with and without core boosting.
We observe that the runtime overhead of proof logging is in all cases quite tolerable, with
average runtime increase
ranging from 14\% to 29\% depending on the algorithm; see Table~\ref{tab:overheads} for details.
There are at most 3 instances (for \pmin without core boosting) that were only
solved without proof logging within the given 1-h time limit.
While this work is \emph{not} focussed on improving proof \emph{checking} but rather realizing
for the first time proof logging in a multi-objective setting,
Table~\ref{tab:overheads}  also includes the proof checking overhead, i.e.,
(proof checking time)/(Scuttle runtime with proof logging), resulting from
checking the Scuttle proofs with the \veripb checker.
With a time limit of 10 hours enforced for the \veripb checker, we
observed that checking takes on average 1--1.5 orders of magnitude more time
than solving the instances with proof logging enabled. It should be noted that similar
observations have been made, e.g., in the realm of \veripb-based
certified MaxSAT preprocessing~\cite{IOTBJMN24CertifiedMaxSATPreprocessing}. Indeed, these observations motivate
seeking improvements to the current runtime performance of the \veripb checker.
We observed that in cases, in particular for \pminimal (see the appendix for more details),
the proof checking overhead appears to somewhat correlate with the number of PD cuts
produced during search.

\begin{figure}[t]
  \centering
  \begin{tabular}{ccc}
    \multicolumn{3}{c}{with core boosting} \\
    \pminimal & \lowerbounding & \bioptsat \\
    \includegraphics{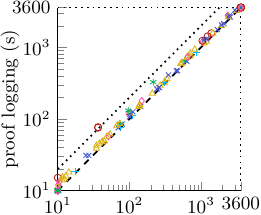} &
    \includegraphics{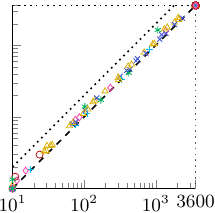} &
    \includegraphics{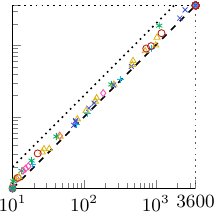} \\
    \multicolumn{3}{c}{without core boosting} \\
    \includegraphics{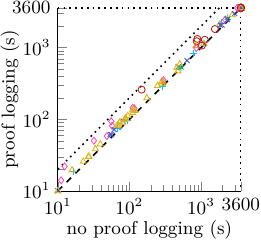} &
    \includegraphics{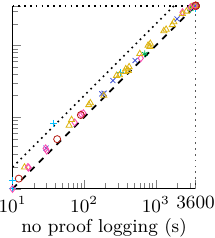} &
    \includegraphics{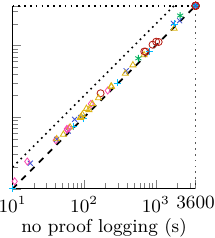} \\
    \multicolumn{3}{c}{\includegraphics{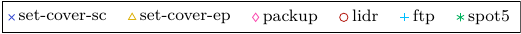}} \\
  \end{tabular}
  \caption{Proof logging runtime overheads.}\label{fig:logging-overhead}
\end{figure}

\section{Conclusions}

We realized for the first time proof logging for multi-objective MaxSAT solving.
Circumventing the fact that \veripb  does not offer direct support for multiple objectives,
we detailed how preorders in \veripb can be used to provide certificates for MO-MaxSAT algorithms
that compute a representative solution for each element of the non-dominated set (with respect to the Pareto order). 
We achieved this 
without changes to the
\veripb format or the proof checker. Integrating  \veripb proof logging into a state-of-the-art multi-objective MaxSAT solver,
we empirically showed that proof logging can be made scalable for MO-MaxSAT.
While we in this work detailed how \veripb can be employed for
proof logging SAT-based multi-objective approaches, the same concepts are applicable
to enabling proof logging for similar algorithmic ideas instantiated for other contexts,
e.g., in the context of pseudo-Boolean optimization.
Developing proof logging methods that capture the computation of all Pareto-optimal solutions, i.e.,
every solution at each element in the non-dominated set, potentially by extending \veripb, %
also remains part of future work.

\begin{credits}
  \subsubsection{\ackname}
We thank Armin Biere, Katalin Fazekas, and Florian Pollitt for their help in
fixing some bugs in CaDiCaL, and Andy Oertel for fixing a bug in the
VeriPB proof checker.
This work is partially funded by
Research Council of Finland (grants 356046 and 362987) and
the European Union (ERC, CertiFOX, 101122653). Views and opinions expressed are however those of the author(s) only and do not necessarily reflect those of the European Union or the European Research Council. Neither the European Union nor the granting authority can be held responsible for them.
\end{credits}

\bibliographystyle{splncs04}
 \bibliography{paper,include/bb_refs_nourl}

\appendix

\newpage

\section{Additional Empirical Data on Proof Checking}

Fig.~\ref{fig:checking-cuts} shows a comparison between the per-instance solving runtimes with proof logging
and the runtime of the \veripb proof checker on the produced proofs.
The color scale represents the range of the number of \pdcuts introduced during solving.
We observe that for \pminimal and \bioptsat without core boosting, a high number of \pdcuts
leads to a higher proof checking overhead, but for the other algorithms
there is no such clear connection.

\begin{figure}[h]
  \centering
  \begin{tabular}{cl}
    \multicolumn{2}{c}{\pminimal} \\
    with core boosting & without core boosting \\
    \includegraphics{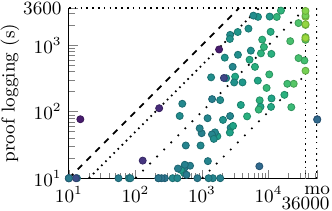} &
    \includegraphics{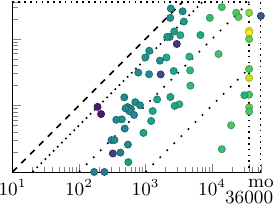} \\
    \multicolumn{2}{c}{\lowerbounding} \\
    \includegraphics{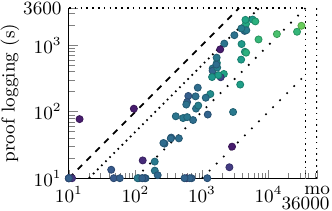} &
    \includegraphics{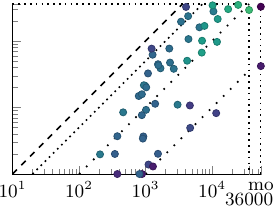} \\
    \multicolumn{2}{c}{\bioptsat} \\
    \includegraphics{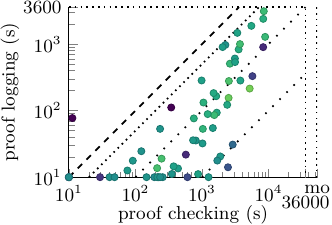} &
    \includegraphics{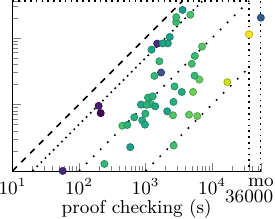} \\
    \vspace{-1em} \\
    \multicolumn{2}{c}{\# \pdcuts: \quad \(10^0\) \includegraphics{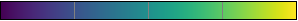} \(10^4\)} \\
  \end{tabular}
  \caption{Runtime comparison between solving with proof logging and proof checking  wrt to the number of \pdcuts.}\label{fig:checking-cuts}
\end{figure}

\end{document}